\documentclass[conference]{IEEEtran}
\IEEEoverridecommandlockouts
\usepackage{cite}
\usepackage{amsmath,amssymb,amsfonts}
\usepackage{algorithmic}
\usepackage{graphicx}
\usepackage{hyperref}
\usepackage{textcomp}
\usepackage{xcolor}

\usepackage{mathptmx}
\usepackage{amsthm}

\def \x {\mathbf{x}}

\newtheorem{thm}{Theorem}

\newtheorem{cor}{Corollary}
\newtheorem{definition}{Definition}

\def\BibTeX{{\rm B\kern-.05em{\sc i\kern-.025em b}\kern-.08em
    T\kern-.1667em\lower.7ex\hbox{E}\kern-.125emX}}
\begin{document}

\title{NPRL: Nightly Profile Representation Learning for Early Sepsis Onset Prediction in ICU Trauma Patients 
\thanks{*Stewart and Hu's research is supported in part by NSF (IIS-2104270). Teredesai's research is supported in part by CueZen Inc. We would like to thank Microsoft Inc., and University of Washington eScience Institute - Azure Cloud Computing Credits for Research and Teaching grant for their support by providing us with the Cloud Compute Resources on Microsoft Azure. All opinions, findings, conclusions and recommendations in this paper are those of the author and do not necessarily reflect the views of the funding agencies.}
}
\IEEEpubid{
    \makebox[\columnwidth]{979-8-3503-2445-7/23/\$31.00 ~\copyright2023 IEEE \hfill}
    \hspace{\columnsep}\makebox[\columnwidth]{}
}

\author{

Tucker Stewart$^1$\quad Katherine Stern$^2$\quad Grant O'Keefe$^2$\quad Ankur Teredesai$^{1,3}$\quad Juhua Hu$^1$\IEEEauthorrefmark{2}\thanks{\IEEEauthorrefmark{2}Corresponding author}  \\
$^1$ School of Engineering and Technology, University of Washington, Tacoma, USA\\
$^2$ Department of Surgery, University of Washington, Seattle, USA\\
$^3$ CueZen Inc., Seattle, USA\\
\{trstew, kstern, gokeefe, ankurt, juhuah\}@uw.edu \\
}

\maketitle

\begin{abstract}
Sepsis is a syndrome that develops in the body in response to the presence of an infection. Characterized by severe organ dysfunction, sepsis is one of the leading causes of mortality in Intensive Care Units (ICUs) worldwide. These complications can be reduced through early application of antibiotics. Hence, the ability to anticipate the onset of sepsis early is crucial to the survival and well-being of patients. Current machine learning algorithms deployed inside medical infrastructures have demonstrated poor performance and are insufficient for anticipating sepsis onset early. Recently, deep learning methodologies have been proposed to predict sepsis, but some fail to capture the time of onset (e.g., classifying patients' entire visits as developing sepsis or not) and others are unrealistic for deployment in clinical settings (e.g., creating training instances using a fixed time to onset, where the time of onset needs to be known apriori). In this paper, we first propose a novel but realistic prediction framework that predicts each morning whether sepsis onset will occur within the next 24 hours using the most recent data collected the previous night, when patient-provider ratios are higher due to cross-coverage resulting in limited observation to each patient. However, as we increase the prediction rate into daily, the number of negative instances will increase, while that of positive instances remain the same. This causes a severe class imbalance problem making it hard to capture these rare sepsis cases. To address this, we propose a nightly profile representation learning (NPRL) approach. We prove that NPRL can theoretically alleviate the rare event problem and our empirical study using data from a level-1 trauma center demonstrates the effectiveness of our proposal.
\end{abstract}


\section{Introduction}
\label{sec:introduction}
Sepsis is a syndrome of a host's dysregulated immune response to the presence of an infection. The host immune response is not specific in its mobilization against pathogens in the body, that it starts to attack the host's own tissue as well, causing organ dysfunction and tissue damage. To this day, sepsis remains to be a prominent complication in the modern medical facilities, particularly ICUs. According to a global audit conducted in \cite{sakr:2018}, depending on the region studied, 13.6\% to 39.3\% of patients admitted in ICU are affected by sepsis. Globally, this proportion is 29.5\%. Of the patients that develop sepsis in the ICU, up to 26\% will die likely due to sepsis. From this, it can be estimated that approximately 7.6\% of patients admitted into ICUs across the globe will die as a result of sepsis, making it one of the leading causes of death in ICUs. Moreover, sepsis also contributes to long term morbidity and mortality. About one-third of ICU sepsis survivors develop persistent and prolonged organ dysfunction, a syndrome commonly recognized as chronic critical illness, characterized by persistent immune suppression~\cite{mira:2017-epidemiology, stortz:2018-benchmarking, stortz:2018-evidence}, muscle wasting\cite{hawkins:2018-chronic}, and recurrent infections~\cite{wang:2014-subsequent}.  Survivors experience high rates of sepsis recidivism, hospital readmission, deficits in physical and cognitive function~\cite{gardner:2019-development}, and increased 1-year and 5-year all-cause mortality~\cite{ou:2016-long}.  

These outcomes can be improved through early intervention \cite{liu:2019}. Early administration of antibiotics treats the underlying infection and prevents the progression of sepsis-related organ dysfunction. It is estimated that as much as 80\% of septic shock-related deaths are preventable with early intervention and chances of survival decrease about 7.6\% each hour that action is not taken \cite{kumar:2006}. Hence, it is critical to the survival and well-being of patients to anticipate the onset of sepsis accurately and early. Although patients who acquire sepsis during hospitalization (i.e., hospital-acquired sepsis) are at greatest risk for sepsis-associated morbidity and mortality~\cite{page:2015-community}, early detection is challenging in these groups. In critically ill trauma patients (i.e., individuals admitted to the ICU for management of injury caused by blunt or penetrating force), injury-related inflammation and organ dysfunction may increase the risk for sepsis, while also masking the clinical signs of infection~\cite{eguia:2019-risk, eriksson:2019-comparison}. Therefore, detecting sepsis early in the critically ill trauma population is in great need but challenging, which is the focus of this work. 

Due to the abundance of available data, many have looked into machine learning using patients' Electronic Health Records (EHRs) to predict sepsis or septic shock. However, many prior studies only work retrospectively~\cite{khoshnevisan:2021}. Specifically, many studies like \cite{futoma:2017-improved, khojandi:2018, liu:2019, scherpf:2019, lauritsen:2020-explainable, ramchand:2020, khoshnevisan:2021} identify the timestamp of the target event such as sepsis onset or septic shock, and then look back in time for a fixed time interval for early prediction. As such, the time length between prediction and the target event stays fixed. This approach is less useful prospectively for deployment because clinicians often do not know in advance when sepsis will actually onset. Therefore, in practice, it would be difficult to decide how and when to use such ML models for prediction. For example, we may need to use the model very frequently (e.g., every hour), so that we do not miss any chance of early prediction.  This is not realistic in a live setting. In addition, some prior approaches like \cite{khojandi:2018, liu:2019, lauritsen:2020-explainable, khoshnevisan:2021} created one data sample per hospital admission, not multiple instances such as for each day the patient is in the hospital. In these cases, the model does not discriminate between days, where sepsis is or is not present in the patient. Thus, they do not capture the time of onset as well, which is a major shortcoming we address in our work.

In this paper, we propose a novel but more realistic prediction setup for deployment in hospitals. Specifically, we use data from each night when patients have limited observations from doctors and re-assess the potential for sepsis onset the following day. In this manner, we position the classifier as a diagnostic tool similarly as lab tests rather than an auxiliary alarm system. Then, instead of labeling the entire visit to have sepsis (i.e., between 13.6\% and 39.3\% in the ICU) or not, each day within the visit is examined for the first potential occurrence of sepsis onset. Time prior to sepsis onset is treated as time at risk. However, sepsis examples become rare compared to negative examples, resulting in a severe class imbalance problem (e.g., $<2\%$ sepsis cases in the level-1 trauma center data). Without any help to address this rare event problem, machine learning models can be easily mislead to predict every example to be negative, which is useless. 

In previous sepsis prediction studies~\cite{khojandi:2018, moor:2019, kaji:2019, lauritsen:2020-early, lauritsen:2020-explainable}, the class imbalance problem was often addressed by simply resampling the training data. For example, random oversampling achieves balance by repeating minority examples, but often leads to overfitting those examples in the minor class. On the other hand, random undersampling removes major examples, but results in information loss. Later, \cite{ramchand:2020} proposed to use ensemble to learn various characteristics from data to alleviate the class imbalance problem, which however is very expensive. More importantly, our class imbalance problem is way more serious. Recently, due to the development of deep neural networks, deep learning has outperformed in various aspects. In terms of sepsis prediction, deep learning architectures such as Recurrent Neural Networks (RNNs)~\cite{kaji:2019, liu:2019, scherpf:2019, ramchand:2020} and Temporal Convolutional Neural Networks (TCNs)~\cite{lauritsen:2020-explainable, moor:2019}, which can capture temporal and sequential patterns significantly outperformed traditional machine learning methods. Moreover, deep learning has also helped advance techniques to address the class imbalance problem~\cite{johnson:2019-survey}. For example, in computer vision, \cite{cui:2019-class} proposed a balanced loss that weights the loss for a particular class inversely proportional to number of instances in that class and \cite{liu2022selfsupervised} showed that self-supervised representations are more robust to class imbalance than supervised representations. However, this has been rarely studied in early sepsis onset prediction, where EHR data are totally different from images in computer vision.

Therefore, to do nightly sepsis prediction resulting in rare sepsis events, we propose to do self-supervised Nightly patient Profile Representation Learning (NPRL). During the representation learning, we aim to capture the unique perspectives of different patients at different nights. We theoretically prove that such representations learned from self-supervised learning can preserve the diversity of patient profiles at different nights, and thus examples from the minor class will not be dominated by those from the major class. Our empirical study conducted on traumatic patients from a level-1 trauma center demonstrates the effectiveness of our proposal. The main contributions of this work can be summarized as follows.

\begin{itemize}
    \item We propose a novel nightly prediction setup for sepsis in ICU trauma patients that is more deploy-able in prospective clinical environments of ICUs and meets the needs of ICU staff.
    \item We develop a self-supervised nightly patient profile representation learning method to capture unique perspectives of different patients at different nights, as to alleviate the severe class imbalance problem with a theoretical guarantee.
    \item We demonstrate the effectiveness of our proposal using data from a level-1 trauma center and compare it with various existing state-of-the-art methods.
\end{itemize}

\section{Related Work}

In this section, we briefly review different sepsis prediction setups, machine learning methodologies used in the literature for early sepsis prediction, and existing techniques applied to address its class imbalance problem.

\subsection{Historical Sepsis Prediction Setup}

How the prediction setup for sepsis is framed has many major consequences for the utility of the predicted label and the model fitting, which has been discussed in depth by Lauritsen et al.~\cite{lauritsen:2021-consequences}. In the reviewed literature, there are three different setups used, that is, fixed time to onset, sequential, and sliding window. 

Most studies used a fixed time to onset framing for their prediction setup \cite{futoma:2017-improved, khojandi:2018, scherpf:2019, lauritsen:2020-explainable, ramchand:2020, khoshnevisan:2021}. For example, \cite{khojandi:2018} predicted onset in 12, 24, or 48 hours prior to onset. However, as mentioned by~\cite{liu:2019}, this setup is learning to predict onset at fixed intervals prior to onset, and thus they are often limited to these retrospective studies, where the time of onset is known a priori. Thereafter, in deployment, these models will need to assess patients' risk at regular interval as not to miss the chance to capture onset. The next most used prediction setup is sequential~\cite{futoma:2017, kaji:2019, moor:2019, lauritsen:2020-early}, which generates predictions at regular intervals using all the data available from admission to the time of prediction. This is more realistic in that it expects to assess patients at regular intervals and it considers all relevant information available. However, a major challenge is the variable observation window. Moreover, the observation window grows the longer the patient stays in the hospital. Thereafter, the amount of data to include can be quite large, and thus the data is usually aggregated to distill the information down.

Sliding window~\cite{kam:2017-learning, nemati:2018-interpretable, van:2019-improving} is similar to sequential, except that it uses a fixed time length prior to prediction for the observation window, making the shape of the input fixed and only considering the most recent data. With this approach, the raw EHR data can be used directly, preserving the temporally progressing patterns. Hence, sliding window approach is preferable for our purpose in discerning when the onset of sepsis occurs. However, there are unmet needs in medical infrastructures. For example, during night hours, the patient-provider ratios are much higher due to cross-coverage, and thus there is limited observations to each patient. An assistant prediction from machine learning in situations when we do not have full observations of each patient would be valuable. In this work, we focus on whether sepsis onset will occur within a 24-hour window with the help of most recent data collected at night. We time delivery of daily predictions for the early morning, right before routine morning rounds, so that medical staff can apply this information to patient care decisions throughout the day.

\subsection{Machine Learning for Sepsis Prediction}

In this subsection, we delve into the machine learning technologies used to predict sepsis and that used to address its class imbalance problem.

\subsubsection{Sepsis Prediction}

In the literature, various traditional machine learning methods have been adopted for sepsis prediction including but not limited to support vector machines~\cite{guillen:2015-predictive}, random forest~\cite{khojandi:2018}, and XGBoost~\cite{barton:2019-evaluation}. A well-known sepsis early warning system currently deployed in hospitals is Epic Sepsis Model (ESM). ESM is a penalized logistic regression model that produces a score indicating the potential of sepsis calculated every fifteen minutes. Recently, \cite{wong:2021} found that ESM performed poorly using data from Michigan Medicine. While there is still some utility of ESM scores to help medical decisions, there is a need for better early predictions of sepsis onset to be deployed in medical infrastructures. It should be noted that we aim to move away from early warning systems that produce frequent false alarms frustrating medical staff.

Recently, due to the advances in deep neural networks and their utility of automatic feature learning, deep learning methods have been increasingly applied to sepsis related prediction. For instance, \cite{futoma:2017, futoma:2017-improved} transformed the raw time-series data through a Multitask Gaussian Process (MGP), and then fed the latent function values through an RNN for the classification, while \cite{moor:2019} extended that to apply the Temporal Convolutional neural Network (TCN). \cite{kaji:2019} also employed a sequential approach using patients' vital signs, demographics, lab results, and medications to predict sepsis onset for the next day. The selected features are aggregated daily so that each day, up to 14 days, is one time step in the RNN. \cite{lauritsen:2020-early, lauritsen:2020-explainable} used an event-based data representation, where each event is formatted as a vector with three values (i.e., timestamp, name of event, and the value) and CNN-RNN or TCN are adopted for sepsis prediction. In this work, we focus more on addressing the severe class imbalance problem. Considering the prevalent promising performance of RNN in capturing temporal patterns for sepsis prediction, we adopt it as our partial deep learning architecture for nightly patient profile representation learning.

\subsubsection{Class Imbalance for Sepsis}

Most existing sepsis prediction studies~\cite{khojandi:2018, kaji:2019, moor:2019, lauritsen:2020-early, lauritsen:2020-explainable, teredesai2022sub, ewig2023multi} address the class imbalance problem in sepsis prediction by simply resampling the data to reduce the class imbalance. However, simple over or under sampling strategies do not solve the problem well~\cite{johnson2019survey}. Random oversampling achieves balance by repeating minority examples, often leading to overfitting on those examples in the minority class, and random undersampling removes majority examples sacrificing much of the available information. To better address the class imbalance problem, \cite{ramchand:2020} utilizes an ensemble of Long Short-Term Memory networks (LSTM). Specifically, multiple LSTM models are trained, capturing different characteristics of the data. Then, their predictions do ensemble to better capture sepsis prediction, which is obviously expensive. 

It should be noted that the development of deep learning has also advanced the machine learning techniques to address the class imbalance problem~\cite{johnson:2019-survey} in other domains but not sepsis. For example, in computer vision, \cite{cui:2019-class} proposed a balanced loss that weights the loss for a particular class inversely proportional to number of instances in that class and \cite{liu2022selfsupervised} showed that self-supervised representations are more robust to class imbalance than supervised representations. However, this has been rarely studied in early sepsis onset prediction, where EHR data are totally different from images in computer vision. Based on our proposed prediction setup, we develop a nightly patient profile representation learning using EHR data to address the serious class imbalance problem.

\section{Nightly Profile Representation Learning}

\begin{figure*}[t]
    \centering
    \includegraphics[width=\textwidth]{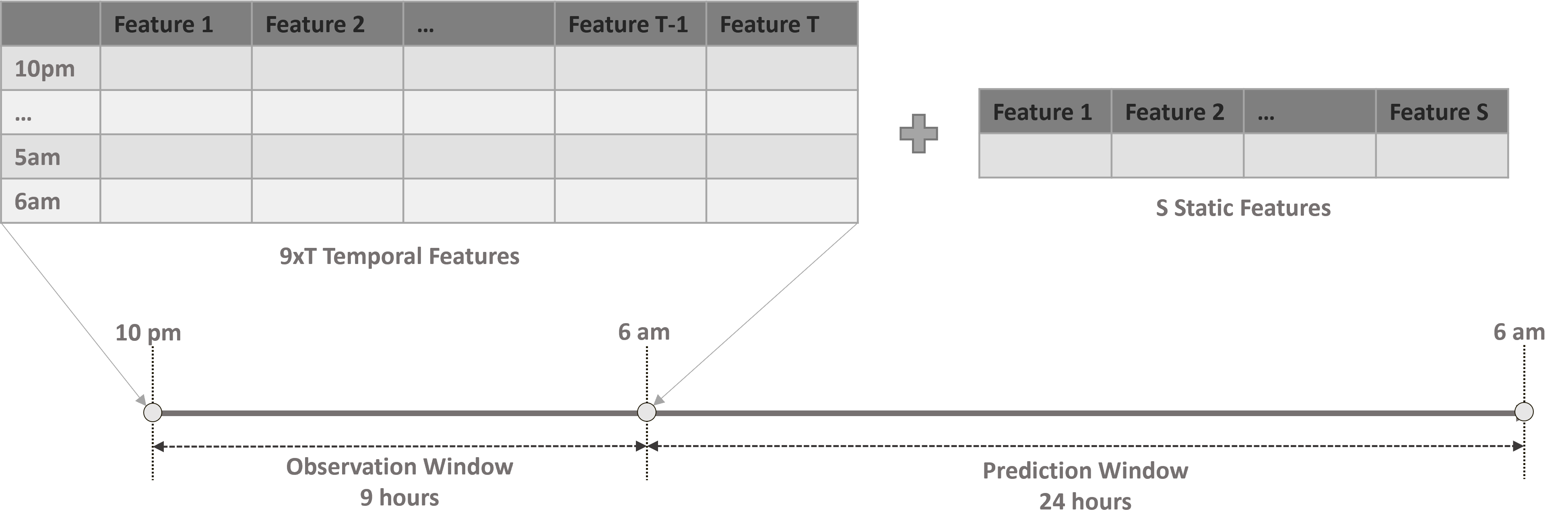}
    \caption{Sepsis Prediction Setup. Temporal features are extracted from the nighttime window between 10 p.m. and 6 a.m., forming a $9\times T$ matrix. $S$ static features which are consistent throughout the admission for each patient. Both are used for predicting sepsis onset within the subsequent 24 hour prediction window. }
    \label{fig:method_framing}
\end{figure*}

In this section, we first describe our nightly prediction setup for early sepsis prediction, which is more realistic in ICUs. Then, due to the resulting severe class imbalance problem making sepsis a rare event, we propose a patient profile representation learning method to capture the unique perspective of different patients at different nights.

\subsection{Early Sepsis Prediction Setup}

In this study, we propose to utilize data collected from patients at night to predict whether the onset of sepsis will occur within the next 24 hours. Specifically, data recorded between the hours of 10 p.m. and 6 a.m. are used to predict whether sepsis onset will occur in the next 24 hours (i.e., until 6 a.m. of the next day). We use data recorded during this nighttime window because 1) most staff are gone for the night so the model captures a gap in observation, 2) data collected at night is more likely to reflect the actual physiology of the patient as they are less exposed to external stimuli of hospital staff, 3) there are fewer interruptions to data collection as diagnostic procedures and interventions are typically planned during daytime hours, and 4) the night window is immediately adjacent to the period of time in which we predict sepsis onset. Each day in a typical ICU, teams of medical staff conduct morning rounds, a standard procedure in which staff will evaluate the current progress of patients and plan treatment for the next 24 hours. This is why we target predictions for early morning. This way, the deployed classifier will act as another decision support diagnostic tool similar to lab tests for staff to discern during rounds. This is much more desirable to clinicians compared to auxiliary monitoring systems that lead to alarm fatigue. Therefore, our setup is more practical and complements current ICU operating procedures. 

It should be noted that for each patient in ICUs, we are able to collect hourly vital signs (e.g., heart rate and body temperature) and hourly cumulative exposures (e.g., IV fluid bolus volume). Given this kind of hourly temporal data, our sepsis prediction problem can be framed as a multi-variate time-series classification task, which is beneficial to adopt RNN to capture temporal patterns for sepsis prediction. Some patient profile information (e.g., age, gender, and injury type) could be potential risk factors contributing to sepsis, but are static. Therefore, in our setting each observation contains 9 hours (i.e., 10 p.m. to 6 a.m.) of $T$ temporal feature values and $S$ static features as illustrated in Fig.~\ref{fig:method_framing}. Then, we aim to predict if sepsis will be onset in the next 24 hours from 6 a.m. of the current day to 6 a.m. the next morning as shown in Fig.~\ref{fig:method_framing}.

\subsection{NPRL for Class Imbalance}
\begin{figure*}[!ht]
    \centering
    \includegraphics[width=\textwidth]{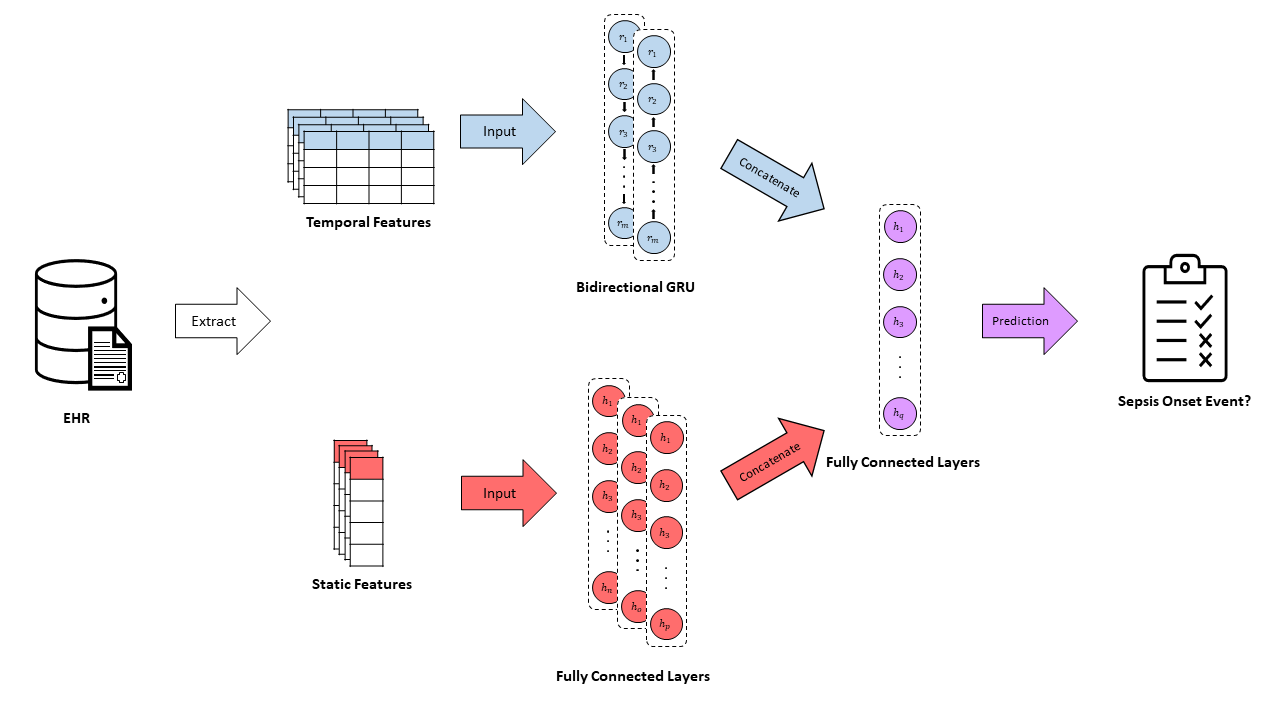}
    \caption{Multi-Modal RNN. Temporal data is processed using a bidirectional GRU layer (blue) to capture temporal information while static data is processed using a deep dense layer architecture (red). Results from each are combined via concatenation for further dense layers (purple).}
    \label{fig:mmr_architecture}
\end{figure*}

Given both the temporal and static data for each observation in our setting, it is natural to adopt a multi-modal RNN for sepsis prediction as illustrated in Fig.~\ref{fig:mmr_architecture}, which is denoted as model $f_\theta$ with parameters $\theta$. Concretely, the $9\times T$ matrix of temporal data is fed to a bidirectional RNN layer to capture the temporal information. LSTM~\cite{hochreiter:1997-long} and Gated Recurrent Unit (GRU)~\cite{cho:2014-learning} are often used, where GRU has been proved to be preferable in medical prediction tasks~\cite{choi:2017-using}. Moreover, bidirectional RNNs can further improve the performance \cite{ma:2017-dipole}. Therefore, we opt for a Bidirectional GRU layer (i.e., the blue component in Fig.~\ref{fig:mmr_architecture}) for our temporal representation learning. Then, the representation learned from the temporal features is concatenated with the representation learned from the $S$ static features using a deep dense layer architecture (i.e., the red component in Fig.~\ref{fig:mmr_architecture}) for the final classification of sepsis. 

Let $\{(x_i^T, x_i^S), y_i\}_{i=1}^n$ denote our training data set, where $x_i^T$ denotes the temporal data, $x_i^S$ denotes the static data, and $y_i\in\{1,\dots, C\}$ is the corresponding label with $C=2$ for sepsis. After feeding our temporal and static data to the multi-modal RNN model $f_\theta$, we can denote the representation in the last fully connected layer before classification as $\x_i=f_\theta(x_i^T,x_i^S)$ (i.e., the purple representation in Fig.~\ref{fig:mmr_architecture}). Then, the classification model can be learned by empirical risk minimization (ERM) as
\[\min_{\theta} \frac{1}{n} \sum_{i=1}^n \ell(\x_i,y_i;\theta)\]
where $\ell(\cdot)$ is a loss function and cross-entropy loss is popular in deep learning. For a $C$-class classification problem, the task is equivalent to
\begin{eqnarray}\label{eq:ori}
\min_{\theta} \sum_{c=1}^C\frac{n_c}{n}g(\x,y,c)
\end{eqnarray}
where $n_c$ shows the number of examples from the $c$-th class and \[g(\x,y,c)=\frac{1}{n_c}\sum_{i:y_i=c} \ell(\x_i,y_i;\theta)\]
According to the formulation, it is obvious that the loss will be dominated by the major class if there is some class $j$ such that $\forall c,c\ne j, n_j\gg n_c$. For example, if we have 100 examples for sepsis prediction with only 2 of them having sepsis, the loss will be dominated by those 98 non-sepsis examples, since we treat each example equivalently. Then, the model may directly predict each example as non-sepsis, which is useless.

Fortunately, in the following, we can show that the self-supervised representation learning can help alleviate the problem. Before we do the sepsis classification directly, we can conduct a self-supervised learning on $f_\theta$ without the sepsis classification layer to initialize the model's parameters as $\theta_0$. Specifically, given $\{(x_i^T, x_i^S)\}_{i=1}^n$ without sepsis assignments, we aim to uniquely identify each observation in the training data, which is corresponding to a night profile for a patient at a specific night and thus named as Nightly Profile Representation Learning (NPRL). Thereafter, $\theta_0$ can be learned by setting the classification layer as a $n$-class classification problem with training data as $\{(x_i^T, x_i^S), i\}_{i=1}^n$, where each instance is uniquely identified by the label $i$.

Given $f_{\theta_0}$, after changing the classification to sepsis, the sepsis classification model can be fine-tuned with a small learning rate and the optimization problem becomes
\begin{eqnarray}\label{eq:pretrain1}
\min_{\theta} \sum_{c=1}^C\frac{n_c}{n}g(\x,y,c) \quad s.t.\quad  \|\theta - \theta_0\|_F\leq \gamma
\end{eqnarray}
which is equivalent to
\begin{eqnarray}\label{eq:pretrain}
\min_{\theta} \sum_{c=1}^C\frac{n_c}{n}g(\x,y,c) + \frac{\lambda}{2}  \|\theta - \theta_0\|_F^2
\end{eqnarray}
Compared with the original objective in Eqn.~\ref{eq:ori}, a regularization term is introduced to constrain the difference between the learned model and the initialized model by NPRL. The benefit from NPRL initialized model can be demonstrated in the following theorem.

\begin{definition}~\cite{sohrab2003basic}
A function $h$ is $L$-Lipschitz continuous if 
\[\|h(x) - h(y)\|_2\leq L\|x-y\|_2\]
\end{definition}

For the sake of convenience, we denote $\mathbf{x}=\theta(x)$ as a simplified definition of $\mathbf{x}=f_{\theta}(x^T, x^S)$. 
\begin{thm}\label{thm:1}
Let $\theta(x)$ be $L$-Lipschitz in $x$ following \cite{DBLP:conf/iclr/SinhaND18} and $\theta^*$ denote the solution of Eqn.~\ref{eq:pretrain}. If $\gamma \leq \frac{1}{8L}$, we have
\begin{align*}
&\forall i,j,\quad \|\theta^*(x_i) - \theta^*(x_j)\|_2^2\\
&\geq\|\theta_0(x_i) - \theta_0(x_j)\|_2^2- \|\theta_0(x_i) - \theta_0(x_j)\|_2/2 - 1/32
\end{align*}
\end{thm}

\begin{proof}
\begin{align*}
&\|\theta^*(x_i) - \theta^*(x_j)\|_2^2\\ 
&= \|\theta^*(x_i) - \theta_0(x_i) + \theta_0(x_j)
- \theta^*(x_j)+\theta_0(x_i) - \theta_0(x_j)\|_2^2\\
&=\|\theta_0(x_i) - \theta_0(x_j)\|_2^2 + 2\langle \theta^*(x_i) - \theta_0(x_i), \theta_0(x_j) - \theta^*(x_j)\rangle\\ 
& +\|\theta^*(x_i) - \theta_0(x_i)\|_2^2 + 2\langle \theta^*(x_i) - \theta_0(x_i), \theta_0(x_i) - \theta_0(x_j)\rangle\\
&+ \|\theta_0(x_j) - \theta^*(x_j)\|_2^2 +2\langle \theta_0(x_j) - \theta^*(x_j),\theta_0(x_i) - \theta_0(x_j)\rangle\\
&\geq \|\theta_0(x_i) - \theta_0(x_j)\|_2^2-2 \|\theta^*(x_i) - \theta_0(x_i)\|_2\|\theta_0(x_j) - \theta^*(x_j)\|_2\\
&-2\|\theta^*(x_i) - \theta_0(x_i)\|_2\|\theta_0(x_i) - \theta_0(x_j)\|_2 \\
&-2\|\theta_0(x_j) - \theta^*(x_j)\|_2 \|\theta_0(x_i) - \theta_0(x_j)\|_2
\end{align*}
Due to the smoothness of $\theta(x)$, we have
\begin{align*}
&\|\theta^*(x_i) - \theta^*(x_j)\|_2^2\\
&\geq \|\theta_0(x_i) - \theta_0(x_j)\|_2^2 - 2 L^2\gamma^2 - 4L\gamma\|\theta_0(x_i) - \theta_0(x_j)\|_2 
\end{align*}
We can easily observe the result with $\gamma \leq \frac{1}{8L}$.
\end{proof}
\textbf{Remark:} The above theorem illustrates that fine-tuning based on a pre-trained model appropriately can preserve the diversity in the pre-trained representations, which can avoid the collapse of minor classes in a class imbalance problem. In our pre-trained model, we aim to uniquely identify each night of each patient. Therefore, for all night instances, we have
\[E_{\forall i,j\in \{1, \dots, n\}, i\neq j}[\theta_0(x_i)^\top \theta_0(x_j)]=0\]
With this observation from NPRL to uniquely identify each night of each patient, we have the bound as follows.
\begin{cor}\label{cor:1}
Let $\theta(x)$ be $L$-Lipschitz in $x$ and $\theta^*$ denote the solution of Eqn.~\ref{eq:pretrain}. If $\gamma \leq \frac{1}{8L}$ and $\theta_0$ is pre-trained with NPRL and $\|\theta_0(x)\|_2=\|\theta^*(x)\|_2=1$, we have
\[E_{\forall i,j\in \{1, \dots, n\},i\neq j}[\theta^*(x_i)^\top \theta^*(x_j)]\leq 0.37 \]
\end{cor}
\begin{proof}
Due to Jensen's inequality, we have
\[E[\|\theta_0(x_i) - \theta_0(x_j)\|_2]\leq \sqrt{E[\|\theta_0(x_i) - \theta_0(x_j)\|_2^2]}=\sqrt{2}\]
Then, we can take it back to the inequality by expanding $\|\cdot\|_2^2$ as
\begin{align*}
&\forall i,j,\quad \|\theta^*(x_i)\|_2^2 + \|\theta^*(x_j)\|_2^2 - 2\theta^*(x_i)^\top \theta^*(x_j)\\
&\geq \|\theta_0(x_i)\|_2^2 + \|\theta_0(x_j)\|_2^2-2\theta_0(x_i)^\top \theta_0(x_j)\\
&- \|\theta_0(x_i) - \theta_0(x_j)\|_2/2 - 1/32
\end{align*}
By applying the expectation on both sides, we have
\begin{align*}
&E_{\forall i,j\in \{1, \dots, n\},i\neq j}[\theta^*(x_i)^\top \theta^*(x_j)]\\
&\leq E_{\forall i,j\in \{1, \dots, n\},i\neq j}[\theta_0(x_i)^\top \theta_0(x_j)]
+ E[\|\theta_0(x_i) - \theta_0(x_j)\|_2]/4\\
&+ 1/64 \leq 0+\sqrt{2}/4+1/64.
\end{align*}
Then, we can easily observe the bound.
\end{proof}
Explicitly, after fine-tuning based on the pre-trained model by NPRL, although the representations are no longer orthogonal in expectation, the diversity can be well preserved. That means the examples from the minor class will not be dominated by those from the major class.

In summary, we propose to do nightly patient profile representation learning, that is NPRL, in which we first conduct a self-supervised representation learning to uniquely identify each night of each patient. Then, the sepsis classification model is initialized by the pre-trained model from NPRL except the classification layer. Finally, the sepsis classification model can be learned by fine-tuning the pre-trained model with a small learning rate to alleviate the class imbalance problem.

\section{Experiments}

To demonstrate the proposed method, we evaluate it on level-1 trauma center data described as follows.

\subsection{Data Description}

We were provided de-identified Electronic Health Record (EHR) data pertaining to patients aged 16 years and older admitted into the ICU of UW Medicine Harborview Medical Center following injury between the years of 2012 and 2019\footnote{E-mail corresponding author for access to this dataset.}. In total, the data contains EHR data for 2,802 patients, 486 of which developed sepsis during their stay in the ICU, making up approximately 17\%. Sepsis is defined according to the 2016 international guidelines, Sepsis-3~\cite{
shankar:2016, 
singer:2016}, as a clinically suspected infection associated with worsening of organ dysfunction. \textbf{We restricted to include only infections that were identified between hospital days 3 and 14, as this is a period of time when patients are at risk for hospital acquired infections~\cite{monegro:2020-hospital} and still in the acute phase of critical illness.} Patients were labeled as having sepsis if they met the following two criteria within three days of each other.
\begin{itemize}
    \item Suspected of having an infection. Infection confirmed with a positive culture sample or by chart review.
    \item Exhibits worsening organ dysfunction denoted by a 2 point or greater increase in Sequential Organ Failure Assessment (SOFA)~\cite{vincent:1996} $(\Delta SOFA \geq 2)$.
\end{itemize}
The time of onset is determined by the time that the positive culture sample was drawn from the patient due to it being the most specific landmark for Harborview Medical Center clinicians' concern that an infection is present or developing.

The provided EHR data includes details about patients' demographics, information about their injuries, physiological signatures recorded hourly, therapeutics administered, and comorbidities. Although information about patient comorbidities were available in the retrospective data and may contribute to sepsis risk, these data are often not ascertained until after the hospital course. Our intention was to develop a model that could easily translate into the clinical setting; so we intentionally omitted data that would not be reliably available at the point of care. From this data, we constructed the features into subsets. Each of these subsets encapsulated one category of data identified as being a contributing factor to the development of sepsis. Subset 1 is vital signs of the patient aggregated each hour, subset 2 contains patient factors and initial physiology from the first 48 hours, and subset 3 is cumulative exposures treated as events accumulated over time. Subsets 1 and 3 are temporal, while subset 2 is static. The purpose of parsing the data into subsets is that subsets can be dynamically isolated and combined to test their contribution to prediction as well as their interaction with one another. Table~\ref{tab:feature_layers} summarizes the detailed feature information for each subset.

\begin{table}[t]
    \centering
        \caption{Feature subsets. * indicates that they are collected from and aggregated over the first 48 hours of the patient's admission.}
    \begin{tabular}{|c|c|}
        \hline
        Subsets & Features \\
        \hline
        Subset 1 & Heart Rate \\
         & Diastolic Blood Pressure (DBP) \\
         & Mean Arterial Pressure (MAP) \\
         & Respiratory Rate \\
         & Temperature \\
         & Fraction of Inspired Oxygen (FiO2) \\
        \hline
        Subset 2 & Age \\
         & Sex \\
         & Mechanism of Injury \\
         & Was Transferred from Another Hospital? \\
         & Has Head Injury? \\
         & First Systolic Blood Pressure in ED \\
         & Reverse Shock Index \\
         & Max Base Deficit* \\
         & Max Lactate* \\
         & Total Red Blood Cell Units Transfused* \\
         & Total Intravenous Crystalloid Volume (L) (Not in OR)* \\
         & APACHE II \\
         & Antibiotic Exposure* \\
         & Number of Surgeries* \\
         & Emergency Department Disposition \\
        \hline
        Subset 3 & IV Fluid Bolus Volume \\
         & Red Blood Cell Units Transfused \\
         & Ventilator Days \\
         & Number of Surgeries \\
         & Surgery Duration \\
        \hline
    \end{tabular}
    \label{tab:feature_layers}
\end{table}

EHR being a large, heterogeneous data format is susceptible to large amounts of missingness and sparsity. Another benefit of our prediction setup is a short observation window. In this case, it is often reliable to apply Last Observation Carry Forward (LOCF) to impute missing values, except for Mean Arterial Pressure that was calculated from Systolic and Diastolic Blood Pressure as $MAP = (2DBP + SBP)/3$ following \cite{sesso:2000-systolic}. Thereafter, SBP is no longer used due to the redundancy, and DBP and MAP being more clinically relevant~\cite{magder:2014-highs}.

\subsection{Instance Extraction and Inclusion}

After cleansing the data and parsing the features into subsets, we extract instances from the nighttime window to train and evaluate our proposed method. For each visit, we extract the temporal features between the hours of 10 p.m. and 6 a.m. for each night the patient is in the ICU for days 3 through 14. Day 1 and 2 are excluded, since hospital acquired infections develop after 48-72 hours of hospitalization~\cite{monegro:2020-hospital}. Infections before day 3 are considered to be acquired from the community or associated with healthcare; therefore it is not a problem of the trauma population. Windows after 14 days are also excluded as after 14 days, patients are no longer considered acutely ill but instead are considered to have chronic critical illness, which has different phenotype from sepsis in the acute recovery phase. In addition, instances containing any null values after LOCF are also excluded. For each visit, the static features are identified and added to each instance.

Lastly, above instances are labeled according to if the first occurrence of sepsis onset occurs within 24 hours after the observation window. While sepsis recidivism is a big problem as well, our setup targets the first sepsis event for the purpose of early prediction. Therefore, instances after the first onset for septic patients were excluded. This results in a total of 25,952 instances with 471 positive ones and 25,481 negative ones, where the serious class imbalance problem can be observed ($<2\%$ for sepsis).

\subsection{Experimental Setup}

To fairly evaluate the performance without bias, 5-fold stratified cross validation was used such that there are five folds of approximately 5,190 instances. Multiple iterations of training and testing are done so that each fold would be used as a test fold once while the other four folds are used for training. The folds are stratified to preserve the original class distribution. In addition, due to the class imbalance between negative and positive instances, the majority class would dominate the empirical risk minimization during the training. In this case, the model would learn to always predict negative. Therefore, during the training of the sepsis vs. non-sepsis classification except for experiments with class-balanced loss, we under-sampled the negative instances and over-sampled the positive instances for the training data only, each with 2,600.

We compare the proposed method NPRL with several existing state-of-the-art methods as follows.
\begin{itemize}
    \item XGBoost~\cite{chen:2016-xgboost}: the state-of-art traditional machine learning method with all available features flattened and concatenated as a long feature vector.
    \item LSTM w/ Attention~\cite{kaji:2019}: an existing RNN model that can do daily prediction but can only use temporal data.
    \item Multi-Modal RNN without NPRL: In GRU, each unit expands the representation of the $T$ features to a $256$-dimensional vector such that the output dimension of this layer is $9$x$512$; The hidden state of each recurrent unit is passed further down the network to incorporate information from each hour of the observation. As such, a flattening layer is used to flatten the hidden state from the recurrent layer to form a 4608-dimensional vector so it can be combined with hidden state from the static component. Static features are processed through a deep neural network comprised of an input layer, followed by three dense layers consisting of 16, 8, and 1 hidden unit(s), respectively. The results of each component are then combined via a concatenation layer that concatenates the hidden states together to form a 4609-dimensional vector that can then be used for classification.
    \item Multi-Modal RNN with Class Balance Loss from \cite{cui:2019-class}.
\end{itemize}

\begin{table*}[t]
    \centering
        \caption{Performance of XGBoost and Multi-Modal RNN using different feature subsets. The results are based on the aggregation (i.e., not average) of 5 test folds. The best is in bold and the 2nd best is underlined.}
    \begin{tabular}{|c|c|c|c|c|c|}
        \hline
        Feature Set & AUROC & True Positives & True Negatives & Sensitivity & Specificity \\\hline
        \multicolumn{6}{|c|}{XGBoost}\\\hline
        Subset 1 & 0.7061 & 287 & 17577 & 0.6093 & 0.6898 \\
        \hline
        Subsets 1+2 & 0.7130 & 296 & 17161 & 0.6285 & 0.6735 \\
        \hline
        Subsets 1+3 & 0.7398 & 315 & 17632 & 0.6688 & 0.6920 \\
        \hline
        Subsets 1+2+3 & 0.7429 & 317 & 17749 & 0.6730 & 0.6966 \\
        \hline
        \multicolumn{6}{|c|}{Multi-Modal RNN}\\
        \hline
        Subset 1 & 0.7097 & 295 & 17710 & 0.6263 & 0.6950 \\
        \hline
        Subsets 1+2 & 0.7100 & 300 & 17187 & 0.6369 & 0.6745 \\
        \hline
        Subsets 1+3 & \textbf{0.7794} & \textbf{346} & \underline{17793} & \textbf{0.7346} & \underline{0.6983} \\
        \hline
        Subsets 1+2+3 & \underline{0.7716} & \underline{332} & \textbf{18108} & \underline{0.7049} & \textbf{0.7106} \\
        \hline
    \end{tabular}
    \label{tab:results_mmrnn}
\end{table*}

\begin{table*}[t]
    \centering
        \caption{Performance comparison using subsets 1 and 3. The results are based on the aggregation (i.e., not average) of 5 test folds. The best is in bold and the 2nd best is underlined.}
    \begin{tabular}{|l|c|c|c|c|c|}
        \hline
        Model & AUROC & True Positives & True Negatives & Sensitivity & Specificity \\
        \hline
        XGBoost & 0.7398 & 315 & 17632 & 0.6688 & 0.6920 \\
        \hline
        LSTM w/ Attention~\cite{kaji:2019} & 0.7610 & \underline{381} & 15222 & \underline{0.8089} & 0.5974 \\
        \hline
        RNN & \underline{0.7794} & 346 & \underline{17793} & 0.7346 & \underline{0.6983} \\
        \hline
        RNN with Class Balanced Loss & 0.7267 & 316 & 17480 & 0.6709 & 0.6860 \\
        \hline
        RNN with Class Balanced Loss (Undersampling) & 0.7733 & 325 & \textbf{18028} & 0.6900 & \textbf{0.7075} \\
        \hline
        RNN with NPRL & \textbf{0.7870} & \textbf{390} & 15602 & \textbf{0.8280} & 0.6120 \\
        \hline
    \end{tabular}
    \label{tab:results_all_layer_1_3}
\end{table*}

For deep learning models, numerical features were each scaled to [0, 1] using Min-Max scaling. Moreover, the static component of the Multi-Modal RNN is excluded from the architecture if no static features are included in the corresponding configuration thus making it just an RNN. Lastly, each of these experiments were evaluated using area under the receiver operating characteristic curve (AUROC). AUROC is often preferred for evaluating binary classification results and is the most used metric to get a holistic understanding of the performance of the classifier. However, it does not provide any information about how the model is performing for positive and negative classes separately. Therefore, we also assess confusion matrix counts along with the sensitivity and specificity. Each of these metrics are calculated for the data in the test fold for each iteration of cross validation. In order to consider all these evaluation metrics extensively, instead of averaging over 5 folds, we aggregated them into one result. For example, when each fold is used as the test data, we have the corresponding number of true positives. Then, the total number of true positives is counted by the summation over all five test folds and the ROC curve is the intercept of each curve. It is mainly due to the reason that averaged values over counts like true positives are hard to interpret. Our code is available here\footnote{\href{https://github.com/ML4UWHealth/NPRL}{
https://github.com/ML4UWHealth/NPRL}}.

\subsection{Performance Comparison}

Before evaluating the efficacy of our proposed methodology, we first assessed the contribution factor each feature subset available in our data to identify an optimal combination by applying two baseline models (i.e., XGBoost and Multi-Modal RNN) to four different feature subset combinations (i.e., Subset 1 only, Subset 1 + Subset 2, Subset 1 + Subset 3, and Subset 1 + Subset 2 + Subset 3).

Table~\ref{tab:results_mmrnn} compares the performance of using different feature subsets. First, it can be observed from the comparison between feature sets that accumulated features in subset 3 can significantly improve the performance for both models. However, the contribution of static features is limited, which may indicate that recent temporal progressing pattern during the night is more useful to predict sepsis onset. Second, RNN with the ability to capture the temporal progressing pattern outperforms XGBoost that treats each feature independently, especially for positive cases. This confirms the effectiveness of temporal models on sepsis prediction as in the literature. Without the loss of common practice, we also report the average performance over the 5 test folds with the standard deviation in Table~\ref{tab:avgf}, in which the similar phenomenon as the aggregated results can be observed.

\begin{table}[ht]
    \centering
        \caption{Average performance of XGBoost and Multi-Modal RNN using different feature subsets. The average performance over 5 test folds with the standard deviation are reported, i.e., avg (std). The best is in bold and the 2nd best is underlined.}
    \begin{tabular}{|c|c|c|c|}
        \hline
        Feature Set & AUROC & Sensitivity & Specificity \\\hline
        \multicolumn{4}{|c|}{XGBoost}\\\hline
        Subset 1 & 0.7061 (0.0265) &  0.6093 (0.0382) & 0.6898 (0.0289) \\
        \hline
        Subsets 1+2 & 0.7130 (0.0265) & 0.6284 (0.0286) & 0.6735 (0.0262) \\
        \hline
        Subsets 1+3 & 0.7398 (0.0297) &  0.6688 (0.0399) & 0.6920 (0.0233) \\
        \hline
        Subsets 1+2+3 & 0.7429 (0.0314) &  0.6730 (0.0412) & 0.6966 (0.0198) \\
        \hline
        \multicolumn{4}{|c|}{Multi-Modal RNN}\\
        \hline
        Subset 1 & 0.7097 (0.0271) & 0.6264 (0.0478) & 0.6950 (0.0721) \\
        \hline
        Subsets 1+2 & 0.7100 (0.0301) & 0.6371 (0.0519) & 0.6745 (0.0757) \\
        \hline
        Subsets 1+3 & \textbf{0.7794 (0.0350)} & \textbf{0.7344 (0.0617)} & \underline{0.6983 (0.0219)} \\
        \hline
        Subsets 1+2+3 & \underline{0.7716 (0.0227)} & \underline{0.7048 (0.0352)} & \textbf{0.7106 (0.0304)} \\
        \hline
    \end{tabular}
    \label{tab:avgf}
\end{table}

\begin{table*}[ht]
    \centering
        \caption{Average performance comparison using subsets 1 and 3. The average performance over 5 test folds with the standard deviation are reported, i.e., avg (std). The best is in bold and the 2nd best is underlined.}
    \begin{tabular}{|l|c|c|c|}
        \hline
        Model & AUROC &  Sensitivity & Specificity \\
        \hline
        XGBoost & 0.7398 (0.0297) &  0.6688 (0.0399) & 0.6920 (0.0233) \\
        \hline
        LSTM w/ Attention~\cite{kaji:2019} & 0.7610 (0.0260) &  \underline{0.8088 (0.0605)} & 0.5974 (0.0705) \\
        \hline
        RNN & \underline{0.7794 (0.0350)} & 0.7344 (0.0617) & \underline{0.6983 (0.0219)} \\
        \hline
        RNN with Class Balanced Loss & 0.7267 (0.0318)  & 0.6710 (0.0417) & 0.6860 (0.0761) \\
        \hline
        RNN with Class Balanced Loss (Undersampling) & 0.7733 (0.0293) &  0.6900 (0.0910) & \textbf{0.7075 (0.0707)} \\
        \hline
        RNN with NPRL & \textbf{0.7876 (0.0198)} &  \textbf{0.8280 (0.0558)} & 0.6123 (0.0772) \\
        \hline
    \end{tabular}
    \label{tab:results_all_avg}
\end{table*}

With the above observations, in the following, we use only feature subsets 1 and 3 to compare the proposal NPRL with other baselines, which is also fair to LSTM w/ Attention~\cite{kaji:2019} that can only use temporal data. Table~\ref{tab:results_all_layer_1_3} shows the results. We can first observe through AUROC that most RNNs can outperform the traditional machine learning method XGBoost with a large margin as observed in the literature by capturing the temporal patterns. Except for RNN with the class-balanced loss from \cite{cui:2019-class}, which is because our problem is too imbalanced to be alleviated by the balanced loss. As well, balanced loss is often more helpful for multi-class classification. This is further demonstrated by additionally applying undersampling at first to bring the majority class to 2,600 instances, which improves the performance of the class-balanced RNN significantly. However, we can observe that the model is biased to non-sepsis cases, where the performance on true negatives is significantly higher but the performance on true positives is quite low. On the contrary, \cite{kaji:2019} significantly improves the performance on true positives for sepsis instances, but sacrifices too much on non-sepsis instances, which results in lower AUROC.

Nevertheless, our proposal with nightly profile presentation learning provides the best AUROC and further improves the sensitivity of the model in predicting the target event significantly, without sacrificing the performance on non-sepsis cases that much. Considering that ICUs prefer to use machine learning models as a pre-screening assistant, they care more about missed positive cases (i.e., never being checked in-depth by clinicians). We can also observe the similar results based on the average performance of the 5 test folds reported in Table~\ref{tab:results_all_avg}. This further demonstrates the efficacy of our proposed method to address serious class imbalance problem for early sepsis prediction.

\section{Conclusion}

In this study, we propose a novel prediction setup for early sepsis onset anticipation by machine learning, which is more applicable and deploy-able as a pre-screening assistant in ICUs. To solve the problem of serious class imbalance in the dataset, we propose to do nightly patient profile representation learning that uniquely identify each patient at each night. Based on such pre-trained model, the diversity between different examples can be highly preserved, so as to alleviate the problem of major class domination. The proposed methodology demonstrates promising results in improving the overall AUROC performance and model sensitivity in early prediction of sepsis, which will assist ICU staff in intervening early. In future work, we will also need to deploy the model and evaluate its applicability in a live setting. 

However, the high sensitivity comes from sacrificing of the specificity, although the highest AUROC is achieved and all trauma patients in the ICU will require attention during morning rounds,  where a false positive will impose only a tiny amount of additional attention. One potential factor is the over and under sampling procedure adopted, where certain negative ones have been ignored. Therefore, a future work is to eliminate the over and under sampling, but still addressing the class imbalance problem using a new loss function inspired by our theoretically analysis from NPRL. Moreover, data augmentation on images is an essential technique for the success of self-supervised learning in computer vision, which will be explored for EHR data as our future work for further improvement. Finally, we were not able to demonstrate our proposal on public data like MIMIC III~\cite{tsiklidis2022predicting} due to the lack of a precise sepsis onset timestamp, whose assignment is a valuable future work. 




\bibliography{ref}
\bibliographystyle{IEEEtran}

\end{document}